\documentclass{article}

\usepackage[preprint]{neurips_2026}
\makeatletter
\providecommand{\@trackname}{} 
\makeatother

\usepackage[utf8]{inputenc}
\usepackage[T1]{fontenc}

\usepackage{url}
\usepackage{booktabs}
\usepackage{amsfonts}
\usepackage{nicefrac}
\usepackage{microtype}
\usepackage[table]{xcolor}
\usepackage{colortbl}
\usepackage{algorithm}
\usepackage{algpseudocode} 
\usepackage{amsmath}
\usepackage{amsthm}
\usepackage{amssymb}
\usepackage{bm}
\usepackage{graphicx}
\usepackage{subcaption}
\definecolor{algoblue}{RGB}{140, 102, 76}
\definecolor{algogreen}{RGB}{34, 120, 102}
\definecolor{algoamber}{RGB}{166, 110, 42}
\definecolor{algogray}{RGB}{110, 118, 129}
\definecolor{algobg}{RGB}{247, 248, 250}

\newcommand{\structstep}[1]{\textcolor{algoblue}{\textbf{#1}}}
\newcommand{\algocommenttext}[1]{\textcolor{algogray}{\footnotesize #1}}

\algrenewcommand{\algorithmiccomment}[1]{\hfill$\triangleright$ \algocommenttext{#1}}

\usepackage{threeparttable}
\usepackage{siunitx}
\usepackage{tikz}
\usetikzlibrary{arrows.meta}

\definecolor{bestbg}{RGB}{210,235,222}
\definecolor{secondbg}{RGB}{235,247,239}
\definecolor{modelbg}{RGB}{244,248,245}

\newcommand{\best}[1]{\cellcolor{bestbg}\textbf{#1}}
\newcommand{\second}[1]{\cellcolor{secondbg}#1}

\definecolor{ModernGreen}{RGB}{46,125,90}
\definecolor{ModernGreenBright}{RGB}{88,166,120}
\definecolor{ModernGreenSoft}{RGB}{236,244,239}

\usepackage[
    colorlinks=true,
    linkcolor=ModernGreen,
    urlcolor=ModernGreenBright,
    citecolor=ModernGreen
]{hyperref}

\newcommand{\rededit}[1]{#1}

\definecolor{grpogreen}{RGB}{46,125,50}
\definecolor{grpored}{RGB}{183,28,28}

\newcommand{\grpoabove}[1]{\textsuperscript{\scalebox{0.9}{\textcolor{grpogreen}{$\uparrow #1$}}}}
\newcommand{\grpobelow}[1]{\textsuperscript{\scalebox{0.9}{\textcolor{grpored}{$\downarrow #1$}}}}

\newtheorem{theorem}{Theorem}
\newtheorem*{theorem*}{Theorem}
\newtheorem{proposition}{Proposition}
\newtheorem{corollary}{Corollary}

\algrenewcommand\algorithmicrequire{\textbf{Input:}}
\algrenewcommand\algorithmicensure{\textbf{Output:}}

\algnewcommand{\LineComment}[1]{\State \(\triangleright\) \textit{#1}}

\theoremstyle{remark}

\newcommand{\mz}[1]{#1}

\title{The Model Knows, the Decoder Finds: Future Value Guided Particle Power Sampling}

\author{
Tu Nguyen$^{1}$\thanks{Equal contribution.} \quad
Matthieu Zimmer$^{2}$\footnotemark[1] \quad
Rasul Tutunov$^{2}$ \quad
Xiaotong Ji$^{2}$ \\
{\bf Haitham Bou Ammar}$^{2,3}$ \\
\\
$^{1}$Huawei Heisenberg Research Center \\
$^{2}$Huawei Noah's Ark Lab \\
$^{3}$UCL Centre for Artificial Intelligence \\
\texttt{\{tu.nguyen, matthieu.zimmer, rasul.tutunov, haitham.ammar\}@huawei.com} \\
\texttt{xiaotong.ji1@h-partners.com}
}
\begin{document}

\maketitle

\begin{abstract}
A recurring pattern in ``reasoning without training'' is that base LLMs already assign non-trivial probability mass to correct multi-step solutions; the bottleneck is locating these modes efficiently at inference time.
A principled way to bias inference toward such modes is \emph{power sampling}, i.e., sampling from $p_\theta(x)^\alpha$ with $\alpha>1$.
Recent work makes power sampling practical by estimating a future-dependent correction factor $z_t$ via Monte Carlo rollouts, thereby replacing iterative Markov chain Monte Carlo with forward-looking estimation.
In this paper, we reframe that correction factor as a \rededit{\emph{future-value selection potential}} in a Sequential Monte Carlo (SMC) view of power sampling: $z_t$ plays the role of a critic-like quantity, but can be estimated directly from the model by short-horizon rollouts, no verifier and no training required.

Building on this view, we introduce \textbf{Auxiliary Particle Power Sampling (APPS)}, a blockwise particle algorithm for training-free reasoning that approximates the sequence-level power target with a bounded population of partial solutions. APPS propagates these hypotheses in parallel by proposal-corrected power reweighting and refines their survival through future-value-guided selection at resampling boundaries, so that finite compute is redistributed across competing prefixes rather than spent along a single unfolding path. This yields a transparent scaling knob in the particle count, predictable peak memory, and a compute pattern that avoids both iterative trajectory editing and dense candidate-wise rollout fan-out, while improving robustness to pivotal early decisions by keeping multiple hypotheses alive throughout decoding. We further study an amortized variant in which the rollout-based selection potential is replaced by a lightweight learned head trained offline from rollout supervision, enabling fast future-value guidance at inference time. More broadly, APPS improves the accuracy--runtime trade-off of training-free decoding, further supporting the view that more faithful inference-time power approximation can recover gains often attributed to post-training.

\end{abstract}

\section{Introduction}

Large language models already contain many correct solutions: short proofs \citep{bai2023qwen}, clean derivations \citep{wei2022chain}, and executable programs that pass tests \citep{chen2021evaluating,li2022competition}.
The bottleneck is increasingly not capability, but reliability under finite inference.
Long contexts \citep{liu2024lost,hsiehruler}, tight latency and memory budgets \citep{wu2025tokenselect,behnam2025rocketkv}, and locally plausible early decisions can quietly derail the remainder of a trajectory \citep{gaoonce,huang2025path}.
This motivates a line of work that \emph{treats reasoning as an inference-time search problem}: without changing model parameters, can we steer generation toward higher-quality trajectories that already have nontrivial probability under the base model \citep{grpo,karan2025reasoning,ji2026scalable}?
\mz{A key insight is that the difficulty is not just in \emph{proposing} good 
continuations, but in \emph{selecting} which partial solutions deserve 
continued computation: a decision that depends on futures the decoder has 
not yet seen.}

\paragraph{From local sharpening to sequence-level targets.}
A natural training-free idea is to sharpen the model distribution.
Low-temperature sampling concentrates mass locally, but it remains myopic: it increases the probability of high-confidence next tokens without correcting for errors that only become visible many steps later.
Sequence-level methods instead target a globally sharpened power distribution over complete trajectories,
$\pi_\alpha(x)\propto p(x)^\alpha$ with $\alpha>1$.
Best-of-$N$ partially approximates this goal by selecting among independent samples, but selection under base-model likelihood is only an indirect proxy for $\pi_\alpha$.
MCMC power sampling targets $\pi_\alpha$ more directly in the limit, but can incur substantial wall-clock overhead because each step repeatedly regenerates suffixes \citep{karan2025reasoning}.
Recent scalable alternatives replace iterative MCMC mixing with explicit lookahead rollouts, improving finite-budget behavior but relocating the burden of inference to decision points, where branching becomes costly as contexts lengthen and batching tightens \citep{ji2026scalable}.
\paragraph{Why future value matters.}
A core difficulty in approximating $\pi_\alpha$ is that the quality of a prefix is not determined by its immediate likelihood alone.
For a prefix $x_{1:t}$, the remaining mass under the power distribution depends on a suffix normalizer
\[
z_t(x_{1:t})\;=\;\sum_{x_{t+1:T}} p(x_{t+1:T}\mid x_{1:t})^\alpha,
\]
whose log is a long-horizon future value: it measures how much power-weighted probability mass remains reachable from the current state.
Prefix-only reweighting can therefore be unstable at finite compute.
It may eliminate prefixes that look slightly worse locally but retain larger $z_t$, and it may over-amplify prefixes that are locally sharp yet lead into dead ends.
\mz{The central question is therefore: can we inject a tractable estimate of 
$z_t$ into finite-budget inference \emph{without} training a separate 
value model, and if so, \emph{where} in the inference algorithm does it 
yield the largest gain per unit of compute?}

\paragraph{Our stance: preserve many presents, and select by reachable futures.}
We take a population-based view in which future value becomes explicit precisely where finite-budget inference must decide what continues to receive mass. Rather than repeatedly unfolding many futures from a single prefix\mz{, which can be costly in \citet{ji2026scalable}}, we maintain a bounded population of prefixes that evolve together and compete over time. At resampling boundaries, selection is guided by reachable power-weighted mass: prefixes with stronger downstream promise retain representation, while collapsed regions relinquish it. In this way, APPS casts finite-budget decoding as a controlled redistribution of probability mass across competing partial trajectories. Its minimal form already instantiates this view; rollout and learned selection potentials refine the same selection law with increasingly explicit estimates of downstream value.

\paragraph{Auxiliary Particle Power Sampling.}
We introduce \textbf{APPS}, a training-free and verifier-free algorithm for approximating power distributions under a finite particle budget. APPS evolves a bounded population of partial trajectories in blocks of tokens, combining proposal-corrected weighting with weighted selection across prefixes. The key observation is that, under the global power target, the merit of a candidate block is not exhausted by its immediate likelihood: it also depends on the power-weighted mass of the continuations that remain reachable after committing to it. APPS is organized around this boundary law. Propagation stays local, while selection carries the long-horizon signal that determines which prefixes survive under finite compute.

\paragraph{Relation to concurrent baselines.}
Our approach is complementary to recent rollout-based and particle-based approximations. Scalable power sampling \citep{ji2026scalable} emphasizes explicit future correction at decision points, while Power SMC \citep{azizi2026power} shows how far a prefix-optimal particle backbone can go when control enters primarily through the proposal. 
\mz{APPS shares the particle perspective but locates its primary 
contribution at a different point in the algorithm: not in the proposal, but in \emph{selection-time future-value guidance}.  
The proposal remains simple and stable; what changes is how 
finite population mass is redistributed at boundaries, informed by 
estimated downstream promise.  This makes the two approaches 
naturally complementary: future-value selection potentials could in 
principle be combined with richer proposals.}

\paragraph{Amortized future value.}
Rollouts make the hidden future of a prefix momentarily visible, but only at substantial cost. We therefore compress this lookahead signal into a lightweight boundary-level head. In this way, APPS need not repeatedly simulate the future in order to sense it: a cheap predictor carries the same downstream signal into selection, while the base language model remains unchanged.

\paragraph{Contributions.}
In summary, we contribute:
\begin{itemize}
\item \textbf{APPS.} A particle-based, verifier-free decoding algorithm that approximates power distributions with a finite inference budget through blockwise proposal-corrected weighting and adaptive resampling.
\item \textbf{Selection-time future value.} A principled use of short-horizon rollouts to estimate the suffix normalizer $z_t$ and inject it only at resampling, refining how finite population mass is redistributed while leaving the proposal-corrected propagation law unchanged.
\item \textbf{Finite-budget population allocation.} A particle-based inference view in which resampling also governs where computation continues to flow over time, coupling effectiveness and efficiency within the same selection law.
\item \textbf{Learned future-value distillation.} A lightweight boundary-level head that amortizes the rollout-based selection potential from hidden states, reducing online rollout cost while preserving the same value-aware selection principle.
\item \textbf{Empirical accuracy--runtime trade-offs.} Across challenging reasoning benchmarks, APPS improves the accuracy--runtime trade-off of training-free decoding, achieving higher accuracy at substantially lower runtime, with both rollout-based and learned future-value selection further enhancing performance.
\end{itemize}

\section{Preliminaries}
\label{sec:prelim}

\paragraph{A distributional view of decoding.}
A pretrained language model does not merely \emph{score} text; it defines a distribution over possible continuations, and decoding is the rule by which we move through that distribution. Fix a prompt $x$, and let $y$ denote a full completion. The model assigns an autoregressive likelihood $p(y\mid x)$, which we may regard as the native plausibility of a trajectory under the base model. Our aim is to change the \emph{walk} without changing the \emph{landscape}: we keep $p$ fixed and modify only the inference procedure, so that trajectories the model already considers plausible become easier to recover under finite computation.

\paragraph{The sequence-level power target.}
To bias decoding toward what the model already believes, we sharpen its distribution at the level of whole trajectories:
\begin{equation}
p^{(\alpha)}(y\mid x) \;\propto\; p(y\mid x)^{\alpha},
\qquad \alpha>1.
\label{eq:power}
\end{equation}
When correct solutions exist as low-probability but genuine modes under $p$, the power target makes them easier to recover by concentrating probability mass around higher-likelihood trajectories. This differs from low-temperature decoding, which sharpens only local conditionals; sequence-level power sampling is global, and therefore depends not only on how good a move looks now, but also on what that move preserves downstream.

\paragraph{Future value at block boundaries.}
Because our method operates in \emph{blocks}, we work at block granularity. Fix a block size $B$ and write a completion as a sequence of blocks $b_{1:J}$. Let $\pi^{(\alpha)}(b_{1:J}\mid x)$ denote the induced sequence-level power target, and for a boundary $j$, let $\pi_j^{(\alpha)}(\cdot\mid x,b_{<j})$ denote the corresponding next-block conditional. Then
\begin{equation}
\pi_j^{(\alpha)}(b_j\mid x,b_{<j})
\;\propto\;
p(b_j\mid x,b_{<j})^\alpha\, z_{j+1}(x,b_{1:j}),
\label{eq:block-conditional}
\end{equation}
where $z_{j+1}(x,b_{1:j})$ is the remaining power-weighted mass of all suffixes after appending $b_j$. This suffix-value term is the exact measure of \emph{future value}: it quantifies how much power-weighted continuation mass remains reachable after committing to the current block. Thus, local power sharpening is not the whole story. The correct boundary law is tilted by what the current choice still keeps alive.

\paragraph{Particle approximation under finite computation.}
The difficulty is that $z_{j+1}(x,b_{1:j})$ is intractable to compute exactly. A natural approximation strategy is therefore population-based: maintain a set of partial trajectories, advance them forward, update their weights, and occasionally resample so that computation remains focused on prefixes that still matter. This yields a blockwise sequential Monte Carlo approximation to the power target. The main failure mode is population collapse: after a few steps, most of the mass may concentrate on only a few particles, leaving the rest effectively irrelevant. We monitor this using effective sample size (ESS). If $\{\bar w^{(i)}\}_{i=1}^P$ are normalized particle weights, then
\begin{equation}
\mathrm{ESS} \;=\; \frac{1}{\sum_{i=1}^P (\bar w^{(i)})^2},
\label{eq:ess-prelim}
\end{equation}
which decreases as the population becomes more concentrated. When ESS falls below a threshold, we resample and reallocate computation toward prefixes carrying the bulk of the mass.

\paragraph{What APPS adds.}
Our method is a blockwise particle approximation to the power target in \eqref{eq:power}. The minimal sampler uses proposal-corrected blockwise weights, so that practical decoding proposals still represent the same sharpened sequence-level target. What remains missing is the future-value term in \eqref{eq:block-conditional}. APPS addresses this by introducing a \emph{selection potential}: a computable proxy for suffix value that enters only through selection, refining which prefixes survive and how offspring are allocated, while leaving proposal and propagation unchanged.

\section{Methodology: Auxiliary Particle Power Sampling}
\label{sec:method}

The minimal requirement for approximating the power target 
$\pi_\alpha(y\mid x)\propto p(y\mid x)^\alpha$ under finite compute is a 
population of partial trajectories whose weights reflect the sequence-level 
objective.  The central difficulty, however, is that prefix likelihood alone 
does not determine a prefix's value under $\pi_\alpha$: the correct 
boundary law also depends on the power-weighted mass of all continuations 
still reachable.  We begin by making this dependence precise, then describe 
the particle backbone that targets the power objective, and finally 
introduce the future-value selection potentials that form the main 
contribution of this work.

\subsection{The power target and its future-value factorization}
\label{sec:future-value-theorem}

Fix a prompt $x$, a finite block horizon $J$, and a countable block 
space.  A full completion is a sequence of blocks 
$b_{1:J}:=(b_1,\dots,b_J)$.  Define the unnormalized sequence-level 
power target
\begin{equation}
  \gamma_J(b_{1:J}\mid x)
  :=
  \prod_{k=1}^{J} p(b_k\mid x,b_{<k})^\alpha,
  \qquad
  \pi_J(b_{1:J}\mid x)
  :=
  \frac{\gamma_J(b_{1:J}\mid x)}{Z_J(x)},
  \label{eq:full-power-target}
\end{equation}
where $Z_J(x):=\sum_{b_{1:J}}\gamma_J(b_{1:J}\mid x)$.  For 
$j\in\{1,\dots,J\}$, define the prefix mass
\[
  \gamma_j(b_{1:j}\mid x)
  :=
  \prod_{k=1}^{j} p(b_k\mid x,b_{<k})^\alpha,
\]
and the suffix value
\[
  z_{j+1}(x,b_{1:j})
  :=
  \sum_{b_{j+1:J}}
  \prod_{k=j+1}^{J} p(b_k\mid x,b_{<k})^\alpha,
\]
with the convention $z_{J+1}\equiv 1$.  The following theorem isolates 
the exact role of future value at every block boundary.

\begin{theorem}[Future-value factorization]
\label{thm:future-value-factorization}
For every block boundary $j$, the $j$-prefix marginal of the power 
target satisfies
\begin{equation}
  \pi_{1:j}(b_{1:j}\mid x)
  =
  \frac{\gamma_j(b_{1:j}\mid x)\,z_{j+1}(x,b_{1:j})}{Z_J(x)}.
  \label{eq:prefix-marginal-factorization}
\end{equation}
Consequently, for every fixed prefix $b_{<j}$, the next-block 
conditional is
\begin{equation}
  \pi_j(b_j\mid x,b_{<j})
  \;=\;
  \frac{p(b_j\mid x,b_{<j})^\alpha\;z_{j+1}(x,b_{1:j})}
       {z_j(x,b_{<j})}\;
  \;\propto\;
  p(b_j\mid x,b_{<j})^\alpha\;z_{j+1}(x,b_{1:j}).
  \label{eq:next-block-conditional}
\end{equation}
\end{theorem}

\begin{proof}
Appendix~\ref{app:future-value-proofs} gives the full derivation.  
The key step is to split the product $\gamma_J$ into a prefix factor 
$\gamma_j$ (constant w.r.t.\ the suffix sum) and a suffix factor 
whose marginal is $z_{j+1}$; dividing adjacent prefix marginals then 
yields \eqref{eq:next-block-conditional}.
\end{proof}

\paragraph{Interpretation.}
Equation~\eqref{eq:next-block-conditional} decomposes the true 
block-boundary law into two factors.  The first, 
$p(b_j\mid x,b_{<j})^\alpha$, is the local power-weighted preference 
for the next block: it is computable from logits alone.  The second, 
$z_{j+1}(x,b_{1:j})$, measures how much power-weighted continuation 
mass remains reachable after committing to that block.  Future value 
is therefore not an auxiliary heuristic: it is exactly the correction 
that turns local power weighting into the true boundary marginal.  
An algorithm that reweights by prefix likelihood alone implicitly sets 
$z_{j+1}\equiv\mathrm{const}$, which can be severely wrong when 
locally strong prefixes lead into regions of low continuation quality.

\paragraph{Finite-budget approximation.}
Theorem~\ref{thm:future-value-factorization} identifies the exact 
boundary law; the next result characterizes how well a finite particle 
population can approximate it.  Consider $P$ particles drawn i.i.d.\ 
from a proposal $q_j$ and reweighted by the self-normalized importance 
ratio $r(b_j) := \pi_j(b_j)/q_j(b_j)$.

\begin{theorem}[Bounded-test particle approximation]
\label{thm:bounded-test-snis}
For every bounded measurable $f$ with $\|f\|_\infty \le 1$, the 
self-normalized importance-sampling estimator 
$\hat\pi_j^P(f) := \sum_{i=1}^P \bar w_i f(X_i)$ satisfies
\begin{equation}
  \mathbb{E}\!\left[\,
    \left|\hat\pi_j^P(f) - \pi_j(f)\right|
  \right]
  \;\le\;
  \frac{4}{\sqrt{P}}\,
  \sqrt{1 + \chi^2(\pi_j \| q_j)}
  \;+\;
  \frac{8}{P}\,\chi^2(\pi_j \| q_j).
  \label{eq:bounded-test-bound}
\end{equation}
\end{theorem}
\noindent\textit{Proof.} See Appendix~\ref{app:bounded-test-snis}.\medskip

\noindent
The key quantity controlling approximation quality is the 
$\chi^2$-divergence $\chi^2(\pi_j\|q_j)$.  When the proposal $q_j$ 
is close to the target $\pi_j$, the divergence is small and moderate 
$P$ suffices; when the two disagree, in particular when $z_{j+1}$ 
varies widely across blocks and the proposal ignores it, the 
divergence grows and more particles are needed.  This sets up the 
central question for Section~\ref{sec:selection-potentials}: can we 
reduce $\chi^2(\pi_j\|q_j)$ without changing the proposal, by 
instead modifying which particles survive at resampling?

\subsection{Blockwise particle backbone}
\label{sec:backbone}

The factorization above (Equation \ref{eq:next-block-conditional}) motivates a two-part approximation strategy: 
a particle population whose weights track the local power factor, and 
a selection mechanism that accounts for the future-value factor.  We 
describe the first part here and the second in 
Section~\ref{sec:selection-potentials}.

APPS maintains $P$ partial trajectories (particles), with time 
measured in blocks of $B$ tokens.  At stage $j$, each particle $i$ 
proposes a next block from a tractable proposal,
\[
  [b_j]^{(i)} \sim q_j(\cdot \mid x,[b_{<j}]^{(i)}),
\]
and receives the incremental log-weight
\begin{equation}
  \Delta \log w_j^{(i)}
  \;=\;
  \alpha \log p\!\bigl([b_j]^{(i)} \mid x,[b_{<j}]^{(i)}\bigr)
  \;-\;
  \log q_j\!\bigl([b_j]^{(i)} \mid x,[b_{<j}]^{(i)}\bigr).
  \label{eq:pps-inc-weight}
\end{equation}
This identity separates \emph{exploration} (the proposal) from 
\emph{representation} (the power target): both log-terms are computed 
from logits already produced during generation, and when the proposal 
uses truncation or temperature, $\log q_j$ is the renormalized 
likelihood of the sampled tokens.  A natural proposal aligns with the 
locally sharpened model (e.g., temperature $\tau < 1$), keeping the 
correction in \eqref{eq:pps-inc-weight} stable.

As noted in Corollary~\ref{cor:minimal-apps-is} 
(Appendix~\ref{app:base-apps}), the cumulative product of these 
incremental weights equals the standard importance ratio for 
$\gamma_J / q$, so the backbone already defines a valid blockwise 
importance-sampling construction for the sequence-level power 
target (even before any future-value correction is introduced).

\paragraph{Resampling and population control.}
After several blocks, particle mass typically concentrates: a few 
prefixes carry most of the weight while many contribute little.  APPS 
monitors this via effective sample size,
\begin{equation}
  \mathrm{ESS}
  \;=\;
  \frac{1}{\sum_{i=1}^{P}(\bar w^{(i)})^2},
  \label{eq:ess}
\end{equation}
and resamples when $\mathrm{ESS} < \kappa P$.  At resampling, 
descendant indices are drawn from the current selection scores 
(Section~\ref{sec:selection-potentials}), and a separate cumulative 
ancestry score is propagated through sampled parent indices for final 
particle selection; this keeps local resampling numerically stable 
while preserving a trajectory-level ranking across resampling events.  
Under \emph{dynamic allocation}, the active population size is 
additionally adapted via a boundary-level ambiguity score 
(Appendix~\ref{app:dynamic-allocation}).

\subsection{Future-value selection potentials}
\label{sec:selection-potentials}

The backbone of Section~\ref{sec:backbone} captures the local power 
factor $p(b_j\mid x,b_{<j})^\alpha$ through proposal-corrected 
weighting.  What it does \emph{not} capture is the future-value factor 
$z_{j+1}$ in the true boundary law 
\eqref{eq:next-block-conditional}.  At finite particle count, this 
omission can cause premature collapse: the population concentrates on 
prefixes that look strong locally but whose continuations are poor.

APPS addresses this by introducing a computable \emph{selection 
potential} $\psi_j^{(i)} \approx z_{j+1}(x, [b_{1:j}]^{(i)})$ that 
enters the algorithm \emph{only} at resampling.  At each boundary 
where resampling is triggered, we form augmented selection weights
\begin{equation}
  \log \tilde w_j^{(i)}
  \;=\;
  \log w_j^{(i)}
  \;+\;
  \eta\,\log\psi_j^{(i)},
  \label{eq:pps-apf-resample}
\end{equation}
where $\eta > 0$ controls the strength of the future-value signal.  
Selection probabilities and expected offspring are then
\begin{equation}
  \rho_j^{(i)}
  :=
  \frac{\tilde w_j^{(i)}}
       {\sum_{k=1}^{P_j}\tilde w_j^{(k)}},
  \qquad
  \mathbb{E}\!\left[N_j^{(i)} 
    \mid \{\tilde w_j^{(k)}\}\right]
  =
  P_j^{+}\,\rho_j^{(i)},
  \label{eq:pps-apf-offspring}
\end{equation}
where $N_j^{(i)}$ is the offspring count and $P_j^{+}$ is the 
post-resampling population size.  Thus future value enters the 
particle system through expected offspring allocation: prefixes with 
stronger downstream promise retain representation, while collapsed 
regions relinquish it.  Proposal and propagation remain unchanged; 
the selection potential acts only through resampling.

\paragraph{Why selection is the right injection point.}
In a finite particle system, future value matters only insofar as it 
changes which branches survive and how many descendants they leave.  
Modifying the \emph{proposal} to depend on $z_{j+1}$ would require 
evaluating the intractable suffix normalizer \emph{before} sampling 
each token: it would be too expensive.  Modifying 
\emph{post-hoc filtering} (e.g., Best-of-$N$) discards all 
intermediate structure.  Selection-time injection is the minimal 
intervention: it preserves the stable, local proposal while 
redirecting finite population mass toward prefixes whose futures 
remain globally promising.

\paragraph{Theoretical benefit: variance reduction from selection 
potentials.}
The following proposition formalizes the intuition that a selection 
potential correlated with the true suffix normalizer reduces the 
particle count needed for a given approximation quality.

\begin{proposition}[Variance reduction from future-value selection]
\label{prop:apf-variance-reduction}
Consider self-normalized importance sampling at block boundary $j$
with $P$ particles, proposal $q_j$, and target
$\pi_j(b_j) \propto p(b_j\mid x,b_{<j})^\alpha\,z_{j+1}(x,b_{1:j})$.
Suppose resampling uses augmented weights
$\tilde w^{(i)} \propto w^{(i)}\,\psi^{(i)}$ with a strictly positive
selection potential $\psi$, and define the effective proposal
\begin{equation}
  \tilde q_j(b_j)
  \;:=\;
  \frac{q_j(b_j)\,\psi(b_j)}
       {\sum_{b'} q_j(b')\,\psi(b')}.
  \label{eq:effective-proposal}
\end{equation}
Then the approximation error at this boundary is governed by
$\chi^2(\pi_j \| \tilde q_j)$ in place of
$\chi^2(\pi_j \| q_j)$, via Theorem~\ref{thm:bounded-test-snis}.
Writing $r(b_j) := \pi_j(b_j)/q_j(b_j)$, we have the exact identity:
\begin{equation}
  1 + \chi^2(\pi_j \| \tilde q_j)
  \;=\;
  \mathbb{E}_{q_j}[r(b_j)^2] \;-\; \operatorname{Cov}_{q_j}\!\!\left(\psi(b_j),\; \frac{r(b_j)^2}{\psi(b_j)}\right).
  \label{eq:chi2-identity}
\end{equation}
Therefore, $\chi^2(\pi_j \| \tilde q_j) < \chi^2(\pi_j \| q_j)$
if and only if
\begin{equation}
  \operatorname{Cov}_{q_j}\!\!\left(\psi(b_j),\; \frac{r(b_j)^2}{\psi(b_j)}\right) \;>\; 0.
  \label{eq:variance-reduction-condition}
\end{equation}
In the ideal limit where $\psi(b_j) \propto r(b_j)$, this covariance attains its theoretical maximum, perfectly canceling the variance term and reducing the $\chi^2$-divergence to $0$.
\end{proposition}

\paragraph{Interpretation.}
The proposition provides a direct mathematical mechanism by which future-value selection \emph{reduces the effective particle requirement}. The bound in Theorem~\ref{thm:bounded-test-snis} scales as $O(P^{-1/2}\sqrt{1+\chi^2})$: replacing the baseline proposal with the augmented effective proposal reduces this penalty by exactly the statistical covariance between the injected potential $\psi$ and the residual discrepancy $r^2/\psi$.

Intuitively, if the potential $\psi$ successfully identifies prefixes with large missing future value ($r$), both terms in the covariance move together: the assigned potential $\psi$ is high, and the remaining ratio $r^2/\psi$ (which scales with $r$) is also high, yielding a strictly positive covariance value. When $\psi$ perfectly matches the optimal trajectory mass $r(b_j) \propto p(b_j)^\alpha z_{j+1}(b_j)/q_j(b_j)$, this covariance perfectly counters the base variance, corresponding to zero-variance adapted importance sampling where $\chi^2(\pi_j\|\tilde q_j) = 0$. Practical estimators operate between these extremes: any surrogate—such as a short rollout—that scales proportionally with the true downstream value $z_{j+1}$ will induce a positive covariance and improve sample efficiency, avoiding the need for an exact internal value model.

\paragraph{Rollout selection potential.}
The first realization constructs $\psi_j^{(i)}$ directly from sampled 
continuations.  For each active particle $i$ at boundary $j$, we 
launch $R$ short lookahead rollouts of horizon $H$ from 
$[b_{1:j}]^{(i)}$, score each by its power-weighted suffix quality, 
and aggregate:
\begin{equation}
  \log\psi_j^{(i)}
  \;\approx\;
  \operatorname{LME}_{r=1}^{R}\; s_{j,r}^{(i)},
  \label{eq:rollout-apf}
\end{equation}
where $\operatorname{LME}$ denotes log-mean-exp and $s_{j,r}^{(i)}$ 
is the rollout suffix score.  This yields a \emph{critic-free} 
estimate of downstream promise: future value is inferred from the 
model's own continuations rather than from a separately trained value 
function.  By Proposition~\ref{prop:apf-variance-reduction}, any 
positive correlation between this estimate and the true $z_{j+1}$ 
reduces approximation error relative to the $\psi\equiv 1$ baseline.

The cost of rollout APF is confined to resampling boundaries:
\[
  O(P_j^{\mathrm{amb}}\,R\,H)
  \quad\text{per boundary},
\]
where $P_j^{\mathrm{amb}} \le P_j$ is the number of ambiguous 
particles selected for lookahead (under dynamic allocation, not all 
particles require evaluation).  The rollouts are short ($H \ll B$), 
independent across particles, and batched; they do not modify the 
proposal or the KV caches used for continued generation.

\begin{figure*}[t]
    \centering
    \includegraphics[width=\textwidth]{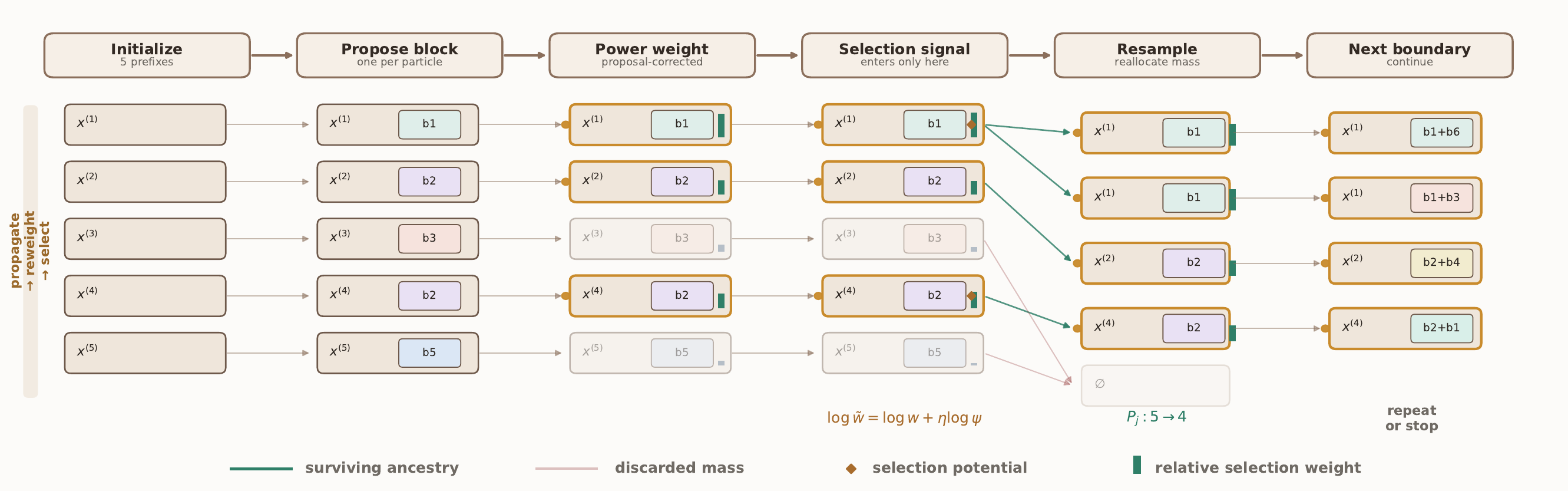}
    \caption{\textbf{Visual overview of APPS at a resampling boundary.}
    Five particle prefixes are propagated blockwise, 
    \textit{reweighted} under the sequence-level power target, and 
    then \textit{selected} under a finite particle budget.  
    \textit{Future-value selection potentials}, when active, refine 
    only the selection weights.  The example also illustrates 
    \textit{dynamic allocation}, with the active population shrinking 
    from $P_j=5$ to $P_{j+1}=4$ before decoding continues.}
    \label{fig:apps-decoding-visual}
\end{figure*}

\paragraph{Learned selection potential.}
The same boundary-level signal can be amortized.  Let 
$h_j^{(i)}\in\mathbb{R}^d$ denote the final-layer hidden state at 
the end of prefix $i$ at boundary $j$.  A lightweight MLP head 
$f_\theta:\mathbb{R}^d\to\mathbb{R}$ predicts
\begin{equation}
  \log\psi_j^{(i)} = f_\theta\!\left(h_j^{(i)}\right).
  \label{eq:learned-apf}
\end{equation}
The head is trained offline to reproduce the within-boundary 
competition induced by rollout APF: which prefixes retain 
representation and in what proportions 
(Appendix~\ref{app:learned-apf} gives training details).  At 
inference time, it replaces explicit lookahead with a single forward 
pass per boundary, reducing the per-boundary cost from 
$O(P_j^{\mathrm{amb}} R H)$ to $O(P_j)$.  Because the head is 
trained to predict the effective rollout signal rather than raw 
suffix values, it preserves the relative ranking that drives 
offspring allocation through \eqref{eq:pps-apf-offspring}, and 
Proposition~\ref{prop:apf-variance-reduction} applies whenever 
its predictions are positively correlated with $z_{j+1}$.

\paragraph{Cost summary.}
The total cost of APPS decomposes cleanly.  Propagation over block $j$ 
costs $O(P_j B)$.  Resampling without a selection potential adds 
negligible overhead.  With rollout or learned potentials, the 
additional cost per triggered boundary is 
$O(P_j^{\mathrm{amb}} R H)$ or $O(P_j)$, respectively.  In 
streaming mode, peak memory is dominated by the KV caches of the 
$P_{\max}$ active particles.  Dynamic allocation 
(Appendix~\ref{app:dynamic-allocation}) further reduces average 
compute by shrinking the active population at boundaries where 
selection has already converged.

\section{Experiments}
\label{sec:experiments}

We study three main questions: (Q1) how APPS compares with recent training-free power-sampling baselines under practical inference budgets; (Q2) when explicit future-value guidance improves over \emph{p-only APPS}, i.e., the base sampler with proposal-corrected power weighting only and no auxiliary future-value selection potential; and (Q3) whether a learned \rededit{selection potential} can approximate the rollout signal. We also study a complementary efficiency question: (Q4) how \emph{dynamic allocation} changes the realized particle population and the resulting runtime--accuracy behavior under a fixed maximum particle cap. Unless noted otherwise, all results are single-seed, full-benchmark evaluations with identical prompts and task evaluators.

\textbf{Benchmarks.}
We evaluate on three complementary reasoning benchmarks: \textbf{MATH500}~\citep{lightman2023letsverify}, scored by exact match on the extracted final answer; \textbf{HumanEval}~\citep{chen2021evaluating}, scored by functional correctness under the standard unit tests; and \textbf{GPQA-diamond}~\citep{rein2024gpqa}, scored by exact match on the extracted answer in $\{A,B,C,D\}$. We use \texttt{pass@1} throughout.

\textbf{Models and APF training data.}
Consistent with prior work, we study three 7B model families: \textit{Qwen2.5-Math-7B}, \textit{Qwen2.5-7B}, and \textit{DeepSeek-Math-7B}. For learned \rededit{selection potentials}, we train a lightweight APF head on 14{,}468 rollout-labeled prompt--answer pairs from a mixed reasoning corpus: 45\% BBH \citep{bhh}, 45\% GSM8K \citep{gsm8k}, and 10\% MATH \citep{mathdataset}, with the MATH portion drawn from the full corpus but excluding MATH500 \citep{lightman2023letsverify}. This yields broad task coverage across reasoning domains while keeping training data disjoint from evaluation. Table~\ref{tab:qwen_math_main} reports completed runs under a common protocol, including \emph{p-only APPS}, APPS with rollout \rededit{selection potential}, and learned \rededit{selection potential} where available.

\textbf{Baselines and protocol.}
We compare against the training-free baselines used in recent power-sampling work: standard decoding, low-temperature decoding, Best-of-$N$ \citep{brown2024large}, MH-MCMC power sampling \citep{karan2025reasoning}, and scalable power sampling (SPS) \citep{ji2026scalable}. As a post-training reference, we report the GRPO numbers from \citet{ji2026scalable}, based on the checkpoint of \citet{shao2025spurious}. For published baselines, we retain the original protocols and numbers.

We additionally show Power SMC \citep{azizi2026power} as reference points in the runtime plots where comparable results are available, but omit it from Table~\ref{tab:qwen_math_main} to keep the main comparison on a uniform grid: the published Power SMC results are concentrated on MATH500 and do not report a matching full HumanEval/GPQA grid for the models and settings evaluated here.

For APPS, we fix a dataset-specific block schedule across all models, variants, and particle budgets: $B=192$ for \textbf{MATH500}, $B=16$ for \textbf{HumanEval}, and $B=64$ for \textbf{GPQA}. This keeps APPS in a comparable practical compute regime to prior power-sampling baselines while reflecting that APPS updates its particle population at block boundaries. Unless noted otherwise, APPS runs use \emph{dynamic allocation}, with $\alpha=4$, $T_{\max}=3072$, EOS termination, particle budgets $P\in\{8,16,32\}$, and a fixed APF strength $\eta=0.5$; $P$ denotes the particle-budget cap. Rollout APF uses two short rollouts per APF call, evaluated only at selected resampling events. When a baseline is reported under a different budget or decoding setting, we retain the published number and mark the mismatch explicitly.

\textbf{APPS variants and ablation structure.}
\mz{We treat the base particle sampler with proposal-corrected power 
weighting only (\textbf{p-only}) as the ablation baseline: it 
validates that the blockwise SMC backbone targets the correct 
sequence-level objective, but does not use any future-value signal 
at selection time. The two future-value variants constitute the 
main experimental conditions: \textbf{APPS + rollout selection 
potential}, which adds short-rollout downstream-value estimates at 
resampling boundaries, and \textbf{APPS + learned selection 
potential}, which replaces explicit lookahead with a lightweight MLP 
head evaluated at the same boundaries. Comparing the future-value 
variants against p-only isolates the effect of the selection 
potential; comparing rollout against learned isolates the effect of 
amortization. Across all three conditions, the base model, prompts, 
evaluators, and decoding setup are fixed; only the resampling weights 
change.}

\subsection{Main results across model families}
\label{sec:main-results-qwen}

Table~\ref{tab:qwen_math_main} summarizes the validated fixed-$P=32$ results under the protocol above and gives the main empirical picture for \textbf{Q1}--\textbf{Q3}.  It reports a common particle-cap slice for all APPS variants, matching the selected operating points highlighted in Figure~\ref{fig:qwen_math_runtime_triptych}.  This fixed-cap view avoids selecting different APPS runs at different values of $P$, and makes the comparison between p-only, rollout APF, and learned APF direct.

\begin{figure*}[t]
  \centering
  \includegraphics[width=\textwidth]{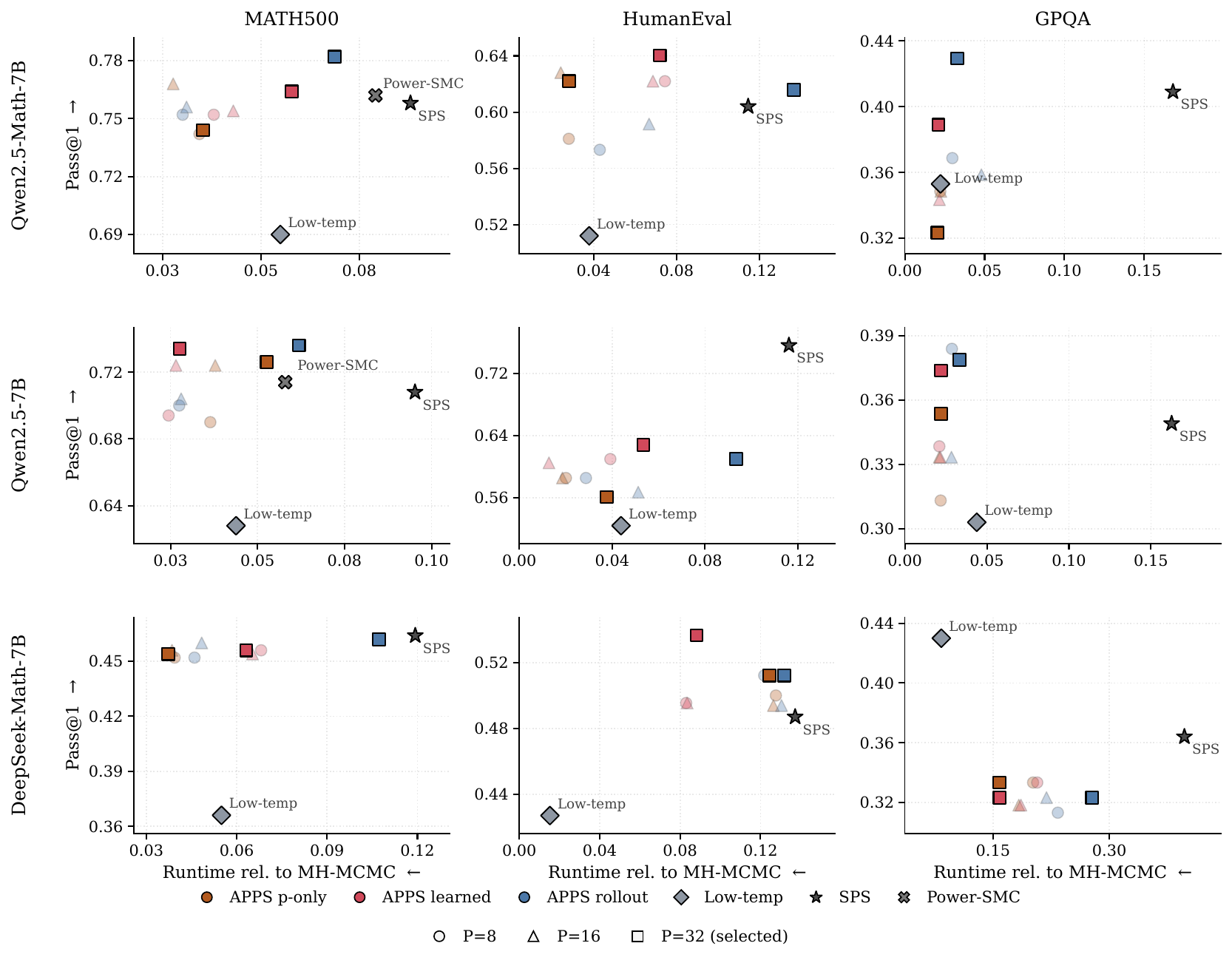}
\caption{\textbf{Runtime--accuracy frontiers across three 7B models.}
Each point shows the highest-\texttt{pass@1} valid run within a method family at particle count $P \in \{8,16,32\}$; square markers denote the selected $P=32$ operating points reported in Table~\ref{tab:qwen_math_main}.
The plot is read as a Pareto-style comparison, with runtime normalized by a local MH-MCMC reference and preferred directions indicated by the axis arrows ($\leftarrow$ lower wall-clock, $\uparrow$ higher \texttt{pass@1}).
Across models, \textbf{APPS (p-only)} is usually the fastest APPS variant, while \textbf{APPS (rollout APF)} often delivers the strongest accuracy, most clearly on MATH500 and in the strongest GPQA runs.
At the same time, \textbf{APPS (learned APF)} frequently tracks rollout closely while remaining more runtime-efficient, making it a practical proxy for rollout guidance on HumanEval and GPQA.
With value-guided selection, dynamic allocation, and finite-sample resampling, increasing the particle cap $P$ need not monotonically improve the runtime--accuracy operating point; the best empirical point can occur below the largest budget.}

\label{fig:qwen_math_runtime_triptych}
\end{figure*}

For \textbf{Q1}, APPS exhibits a strong training-free accuracy--runtime profile across the three model families.  In the fixed-$P=32$ table, APPS attains the best training-free result on all three Qwen2.5-Math-7B benchmarks, on MATH500 and GPQA for Qwen2.5-7B, and on HumanEval for DeepSeek-Math-7B.  The remaining slices are led by SPS or low-temperature decoding, but the runtime frontiers in Figure~\ref{fig:qwen_math_runtime_triptych} show that APPS consistently occupies favorable Pareto regions: its operating points combine high \texttt{pass@1} with small MH-MCMC-normalized wall-clock cost.  Within APPS, p-only provides the simplest low-overhead variant, while APF-guided variants often supply the strongest accuracy.  Thus, APPS is not only competitive at a fixed particle budget; it provides practical training-free operating points that improve the accuracy--runtime frontier in many tested settings.

For \textbf{Q2}, future-value guidance improves finite-particle selection, but the size of the gain depends on the model and task.  Rollout APF gives the strongest APPS accuracy on several reasoning-heavy slices, including Qwen2.5-Math-7B on MATH500 and GPQA, Qwen2.5-7B on MATH500 and GPQA, and DeepSeek-Math-7B on MATH500.  Learned APF also improves over p-only in important settings, most clearly on HumanEval for Qwen2.5-Math-7B and DeepSeek-Math-7B, and remains competitive on many MATH500 and GPQA slices.  These results support the central premise of APPS: when local likelihood is an imperfect proxy for downstream correctness, injecting future-value information at resampling can improve the finite-particle approximation.

For \textbf{Q3}, learned APF is best interpreted as an amortized proxy for rollout guidance, not as a uniformly stronger oracle.  It often tracks rollout APF closely while avoiding repeated online lookahead, and in Figure~\ref{fig:qwen_math_runtime_triptych} occupies favorable runtime--accuracy regions, especially on HumanEval and GPQA.  This indicates that the learned selection potential captures a substantial part of the rollout signal at lower online cost.  However, rollout APF still reaches the highest APPS accuracy in several MATH500 and GPQA settings, showing that explicit online estimates can remain valuable when their finite-horizon signal is well aligned with final correctness.

Taken together, the three model families suggest a task-dependent but consistent hierarchy.  P-only anchors the speed-oriented end of the APPS family and remains highly competitive when inference compute is constrained.  Rollout APF can deliver the highest accuracy, especially on MATH500 and in the strongest GPQA runs, but its gains are less uniform because it relies on finite-sample, finite-horizon estimates evaluated only at selected resampling events. Learned APF is generally the more practical future-value mechanism: it retains much of the benefit of rollout guidance while reducing online cost and exposure to rollout noise. The DeepSeek results are more mixed than the Qwen results.  One likely factor is completion length: on GPQA, DeepSeek outputs are much shorter than Qwen outputs, with a median length of only 27 characters versus roughly 1.5k for Qwen2.5-7B and 1.8k for Qwen2.5-Math-7B, leaving limited room for branching and resampling to affect the final answer.

Finally, \textbf{Q4} provides a complementary view of particle allocation and efficiency. 
In Figure~\ref{fig:qwen_math_runtime_triptych}, APPS runs use dynamic allocation by default and occupy favorable realized-cost regions of the runtime--accuracy frontier: the methods remain competitive in \texttt{pass@1} while operating at small MH-MCMC-normalized wall-clock cost. 
The same plot also shows that increasing the particle cap $P$ does not always monotonically improve the observed operating point, consistent with an adaptive finite-particle regime in which the cap sets available capacity but the realized active population $P_j$ determines how much computation is actually spent. 
This makes the runtime frontier a realized-cost frontier: the statistical target and proposal-corrected weighting remain fixed, while dynamic allocation determines how much of the particle cap is actually used across generation blocks. 
Appendix Figure~\ref{fig:rollout_learned_dynamic_tradeoff} further probes this behavior for APF-guided variants by comparing learned APF with dynamic allocation against rollout APF with and without dynamic allocation.


\newcommand{\tcell}[1]{\multicolumn{1}{c}{#1}}

\begin{table*}[t]
\centering
\caption{\textbf{Training-free sampling results across three 7B models at fixed particle budget $P=32$.}
We compare Qwen2.5-Math-7B, Qwen2.5-7B, and DeepSeek-Math-7B on MATH500, HumanEval, and GPQA.
Best and second-best \emph{training-free} entries within each model--dataset pair are highlighted; APPS rows correspond to the $P=32$ operating points shown in the runtime--accuracy figures.}
\label{tab:qwen_math_main}

\scriptsize
\setlength{\tabcolsep}{6pt}
\renewcommand{\arraystretch}{1.10}

\begin{threeparttable}
\begin{tabular}{
l
S[table-format=1.3]
S[table-format=1.3]
S[table-format=1.3]
}
\toprule
\textbf{Method} & {\textbf{MATH500}} & {\textbf{HumanEval}} & {\textbf{GPQA}} \\
\midrule

\rowcolor{modelbg}
\multicolumn{4}{l}{\textbf{Qwen2.5-Math-7B}} \\
Base decoding \citep{ji2026scalable} & 0.496 & 0.329 & 0.278 \\
Low-temp ($\tau=1/\alpha$) \citep{ji2026scalable} & 0.690 & 0.512 & 0.353 \\
Best-of-$N$ ($N=32$) \citep{ji2026scalable} & 0.684 & 0.512 & 0.343 \\
MH-MCMC power sampling \citep{karan2025reasoning} & 0.748 & 0.573 & 0.389 \\
Scalable power sampling (SPS) \citep{ji2026scalable} & 0.758 & 0.604 & \tcell{\second{0.409}} \\
\cmidrule(lr){1-4}
\textbf{APPS} (\textit{rollout} APF) & \tcell{\best{0.782}} & 0.616 & \tcell{\best{0.429}} \\
\textbf{APPS} (\textit{learned} APF) & \tcell{\second{0.764}} & \tcell{\best{0.640}} & 0.389 \\
\textbf{APPS} (p-only) & 0.744 & \tcell{\second{0.622}} & 0.323 \\
\cmidrule(lr){1-4}
GRPO (MATH) \citep{shao2025spurious}
& \tcell{0.785\grpobelow{0.3}}
& \tcell{0.537\grpoabove{10.3}}
& \tcell{0.399\grpoabove{3.0}} \\

\midrule

\rowcolor{modelbg}
\multicolumn{4}{l}{\textbf{Qwen2.5-7B}} \\
Base decoding \citep{ji2026scalable} & 0.498 & 0.329 & 0.278 \\
Low-temp ($\tau=1/\alpha$) \citep{ji2026scalable} & 0.628 & 0.524 & 0.303 \\
Best-of-$N$ ($N=32$) \citep{ji2026scalable} & 0.650 & 0.609 & 0.282 \\
MH-MCMC power sampling \citep{karan2025reasoning} & 0.706 & 0.622 & 0.318 \\
Scalable power sampling (SPS) \citep{ji2026scalable} & 0.708 & \tcell{\best{0.756}} & 0.349 \\
\cmidrule(lr){1-4}
\textbf{APPS} (\textit{rollout} APF) & \tcell{\best{0.736}} & 0.610 & \tcell{\best{0.379}} \\
\textbf{APPS} (\textit{learned} APF) & \tcell{\second{0.734}} & \tcell{\second{0.628}} & \tcell{\second{0.374}} \\
\textbf{APPS} (p-only) & 0.726 & 0.561 & 0.354 \\
\cmidrule(lr){1-4}
GRPO (MATH) \citep{shao2025spurious}
& \tcell{0.740\grpobelow{0.4}}
& \tcell{0.561\grpoabove{19.5}}
& \tcell{0.354\grpoabove{2.5}} \\

\midrule

\rowcolor{modelbg}
\multicolumn{4}{l}{\textbf{DeepSeek-Math-7B}} \\
Base decoding \citep{ji2026scalable} & 0.362 & 0.415 & 0.333 \\
Low-temp ($\tau=1/\alpha$) \citep{ji2026scalable} & 0.366 & 0.427 & \tcell{\best{0.430}} \\
Best-of-$N$ ($N=32$) \citep{ji2026scalable} & 0.420 & 0.433 & 0.338 \\
MH-MCMC power sampling \citep{karan2025reasoning} & 0.424 & 0.470 & 0.345 \\
Scalable power sampling (SPS) \citep{ji2026scalable} & \tcell{\best{0.464}} & 0.487 & \tcell{\second{0.364}} \\
\cmidrule(lr){1-4}
\textbf{APPS} (\textit{rollout} APF) & \tcell{\second{0.462}} & \tcell{\second{0.512}} & 0.323 \\
\textbf{APPS} (\textit{learned} APF) & 0.456 & \tcell{\best{0.537}} & 0.323 \\
\textbf{APPS} (p-only) & 0.454 & \tcell{\second{0.512}} & 0.333 \\
\cmidrule(lr){1-4}
GRPO (MATH) \citep{shao2025spurious}
& \tcell{0.492\grpobelow{2.8}}
& \tcell{0.524\grpoabove{1.3}}
& \tcell{0.333\grpoabove{9.7}} \\

\bottomrule
\end{tabular}

\begin{tablenotes}[flushleft]
\footnotesize
\item Highlighting is over \emph{training-free} methods only.
\item APPS rows use the selected $P=32$ operating points from Figure~\ref{fig:qwen_math_runtime_triptych}.
\item GRPO is a post-training reference; each annotation reports the absolute \texttt{pass@1} difference (in percentage points) between GRPO and the best training-free method in that block, with the arrow indicating whether the best training-free method is above or below GRPO.
\end{tablenotes}
\end{threeparttable}
\end{table*}

\section{Conclusion}

We introduced APPS, a training-free framework for particle power sampling that brings post-training-style gains closer to base-decoding efficiency. By separating proposal-corrected propagation from selection-time future-value guidance, APPS yields a strong minimal sampler that already operates near the speed of base decoding, while rollout and learned selection potentials further improve how finite inference-time compute is deployed. Empirically, this places APPS in the narrow but important regime between decoding and post-training: substantially more accurate than standard decoding, often competitive with or stronger than recent training-free baselines, and in selected settings surprisingly close to post-trained performance without modifying model weights. More broadly, our results suggest that some of the benefits often attributed to post-training can instead be recovered at inference time, provided the decoder is allowed to maintain, weight, and selectively preserve multiple possible futures.

\newpage
\bibliographystyle{plainnat}
\bibliography{main}

\newpage
\appendix
\section{Extended Related Work}
\label{sec:related}

\paragraph{Post training and verifier based reasoning.}
A major driver of recent reasoning gains is post training with preference \citep{ouyang2022training} or reward optimization \citep{grpo,huang2025tree,yu2025dapo}.
In parallel, verifier based pipelines improve reliability by sampling multiple solutions and selecting using an outcome reward model or a process reward model, with the strongest variants relying on step level supervision and careful reward design.
These approaches can be highly effective, but they introduce training pipelines and often depend on additional supervision signals or verifier models.

\paragraph{Training free inference time search for reasoning.}
A complementary line of work improves reasoning without updating the base model by changing how we sample.
Local sharpening methods \citep{karan2025reasoning} such as low temperature sampling are nearly free but remain myopic, while selection based methods such as Best-of-N \citep{brown2024large,sun2024fast} improve by exploiting sample diversity but still score candidates using the base model.
More structured inference time search methods explore multiple trajectories through sampling and selection \citep{song2025rekg,feng2024alphazero,xie2023s}.
This paper belongs to the training free direction, and targets sequence level sharpening rather than token level heuristics.

\paragraph{Power distributions and finite budget approximations.}
Power sampling \citep{karan2025reasoning} formalizes sequence level sharpening by targeting $\pi_\alpha(x) \propto p(x)^\alpha$.
Early practical realizations use Markov chain Monte Carlo updates over long trajectories, which can be statistically principled but typically incur large wall clock overhead due to repeated suffix regeneration.
Recent scalable approximations \citep{ji2026scalable} replace iterative mixing with explicit lookahead and rollout estimation of future dependent factors, improving finite budget behavior but shifting cost toward candidate wise fan out at decision points.
Closely related work, such as Power-SMC \citep{azizi2026power}, uses sequential Monte Carlo with resampling and cache reordering to achieve strong accuracy--latency trade-offs under large particle counts and shorter generation caps in latency-focused settings. Our method adopts the same particle perspective, but differs in where control enters inference: Power SMC centers the proposal, whereas APPS centers weighted selection under a finite population budget. In its minimal form, APPS already defines a strong blockwise particle sampler; richer \rededit{future-value selection potentials} then refine resampling without altering the underlying propagation law.

\paragraph{Twisted and value guided particle methods for LLMs.}
Several recent works \citep{zhao2024probabilistic,fengstep} bring twisted sequential Monte Carlo ideas into LLM reasoning, often by introducing stepwise potentials that encourage trajectories with higher expected downstream reward or verifier scores.
\cite{fengstep} applies twisted SMC to math reasoning by estimating expected future rewards of partial solutions and using this signal to focus sampling on promising candidates.
These methods highlight the importance of future value signals for long horizon reasoning.
Our approach uses a related future-aware principle, but differs in where the signal enters: propagation remains proposal-aware and stable, while future value acts only through resampling. This lets the same inference law range from a strong p-only backbone to more explicit rollout or learned selection-time value estimates.

\section{Algorithm}

\noindent
APPS implements the sequence-level power objective as a blockwise particle procedure. At each boundary, particles are propagated with a tractable proposal and assigned proposal-corrected power weights, giving a valid importance-weighted approximation of the sharpened target. Resampling then serves as the point where finite computation is redistributed across prefixes. In the p-only variant, this redistribution uses the proposal-corrected weights alone; in rollout and learned APF variants, the same selection step is augmented by an estimated future-value potential, while the propagation rule and target weighting remain unchanged.

\begin{algorithm}[h]
\colorbox{algobg}{%
\parbox{0.98\linewidth}{%
\caption{\textbf{Auxiliary Particle Power Sampling (APPS).}
A blockwise particle sampler for the sequence-level power objective. APPS propagates bounded prefixes, applies proposal-corrected power weighting, and redistributes particle mass by ESS-gated resampling. Rollout and learned APF variants augment only the selection weights with an estimated future-value potential.}
\label{alg:pps-apf}
\begin{algorithmic}[1]
\Require prompt $x_0$, base LM $p$, power $\alpha>1$, particle cap $P_{\max}$, block size $B$, proposal family $\{q_j(\cdot\mid\cdot)\}$
\Require APF mode $\in\{\texttt{none},\texttt{rollout},\texttt{learned}\}$; rollout budget $(R,H)$ if \texttt{rollout}
\Ensure final completion by weighted sampling or score-based selection

\State \structstep{Initialize}
\State Choose active population size $P_1 \leq P_{\max}$
\State Set $x^{(i)} \gets x_0$ and $\log w^{(i)} \gets 0$ for $i=1,\dots,P_1$
\State Set $j \gets 1$

\While{not terminated}

  \State \structstep{Propagate}
  \For{$i \gets 1$ to $P_j$}
    \State $x_{\mathrm{prev}}^{(i)} \gets x^{(i)}$
    \State Sample block $[b_j]^{(i)} \sim q_j(\cdot \mid x^{(i)})$
    \State $x^{(i)} \gets x^{(i)} \oplus [b_j]^{(i)}$
  \EndFor

  \State \structstep{Reweight}
  \For{$i \gets 1$ to $P_j$}
    \State $\log w^{(i)} \gets \log w^{(i)}
      + \alpha \log p([b_j]^{(i)} \mid x_{\mathrm{prev}}^{(i)})
      - \log q_j([b_j]^{(i)} \mid x_{\mathrm{prev}}^{(i)})$
  \EndFor

  \State \structstep{Score selection}
  \State Set $\log \tilde w^{(i)} \gets \log w^{(i)}$ for all $i=1,\dots,P_j$
  \If{APF mode $\neq \texttt{none}$ \textbf{and} resampling is imminent}
    \If{APF mode $= \texttt{rollout}$}
      \State Estimate $\log \psi^{(i)}$ using $R$ rollouts of horizon $H$ from $x^{(i)}$
    \ElsIf{APF mode $= \texttt{learned}$}
      \State Predict $\log \psi^{(i)}$ from $h_{\mathrm{last}}(x^{(i)})$
    \EndIf
    \State $\log \tilde w^{(i)} \gets \log \tilde w^{(i)} + \eta\,\log \psi^{(i)}$ for all $i=1,\dots,P_j$
  \EndIf

  \State \structstep{Allocate and resample}
  \State Set $P_{j+1} \leq P_{\max}$ from the boundary ambiguity score
  \State Compute $\mathrm{ESS}$ from $\{\log \tilde w^{(i)}\}_{i=1}^{P_j}$
  \If{$\mathrm{ESS} < \kappa P_j$}
    \State Sample ancestors $a_{1:P_{j+1}} \sim \textsc{Resample}(\propto \exp(\log \tilde w^{(i)}))$
    \State Optionally preserve an elite particle in the resampled set
    \State Set $x^{(i)} \gets x^{(a_i)}$ and $\log w^{(i)} \gets 0$ for $i=1,\dots,P_{j+1}$
  \EndIf

  \State $j \gets j+1$
\EndWhile

\State \Return final completion by sampling from $\exp(\log w)$ or selecting the best score
\end{algorithmic}
}}
\end{algorithm}

\newpage

\section{\rededit{Proofs for Section~\ref{sec:method}}}

\subsection{Proof of Theorem~\ref{thm:future-value-factorization} 
(Future-value factorization)}
\label{app:future-value-proofs}

\begin{proof}
\rededit{By definition,}
\[
\rededit{\pi_J(b_{1:J}\mid x)
=
\frac{1}{Z_J(x)}
\prod_{k=1}^J p(b_k\mid x,b_{<k})^\alpha.}
\]
\rededit{Fix a boundary $j\le J$. Then}
\[
\rededit{\pi_{1:j}(b_{1:j}\mid x)
=
\sum_{b_{j+1:J}} \pi_J(b_{1:J}\mid x)
=
\frac{1}{Z_J(x)}
\sum_{b_{j+1:J}}
\prod_{k=1}^J p(b_k\mid x,b_{<k})^\alpha.}
\]
\rededit{Split the product into prefix and suffix terms:}
\[
\rededit{\prod_{k=1}^J p(b_k\mid x,b_{<k})^\alpha
=
\left(\prod_{k=1}^j p(b_k\mid x,b_{<k})^\alpha\right)
\left(\prod_{k=j+1}^J p(b_k\mid x,b_{<k})^\alpha\right).}
\]
\rededit{The prefix factor is constant with respect to the sum over $b_{j+1:J}$, so}
\[
\rededit{\pi_{1:j}(b_{1:j}\mid x)
=
\frac{1}{Z_J(x)}
\left(\prod_{k=1}^j p(b_k\mid x,b_{<k})^\alpha\right)
\sum_{b_{j+1:J}}
\prod_{k=j+1}^J p(b_k\mid x,b_{<k})^\alpha
=
\frac{\gamma_j(b_{1:j}\mid x)\,z_{j+1}(x,b_{1:j})}{Z_J(x)},}
\]
\rededit{which proves \eqref{eq:prefix-marginal-factorization}. For the conditional law, divide the prefix marginal at level $j$ by the one at level $j-1$:}
\[
\rededit{\pi_j(b_j\mid x,b_{<j})
=
\frac{\pi_{1:j}(b_{1:j}\mid x)}{\pi_{1:j-1}(b_{<j}\mid x)}
=
\frac{\gamma_j(b_{1:j}\mid x)\,z_{j+1}(x,b_{1:j})}{\gamma_{j-1}(b_{<j}\mid x)\,z_j(x, b_{<j})}.}
\]
\rededit{Since \(\gamma_j(b_{1:j}\mid x)=\gamma_{j-1}(b_{<j}\mid x)\,p(b_j\mid x, b_{<j})^\alpha\), this gives \eqref{eq:next-block-conditional}.}
\end{proof}

\subsection{Proof of Theorem~\ref{thm:bounded-test-snis} 
(Bounded-test particle approximation)}
\label{app:bounded-test-snis}


\begin{proof}
To simplify notation, write $\pi:=\pi_j$ and $q:=q_j$.
Let
\[
r(x):=\frac{d\pi}{dq}(x).
\]
Then
\[
\mathbb E_q[r(X)] = 1,
\qquad
\mathbb E_q[r(X)^2] = 1 + \chi^2(\pi\|q).
\]

Let $X_1,\dots,X_P \overset{\text{i.i.d.}}{\sim} q$, set
\[
r_i := r(X_i),
\qquad
S := \sum_{i=1}^P r_i,
\qquad
\bar w_i := \frac{r_i}{S}.
\]
Fix any measurable $f$ with $\|f\|_\infty\le 1$, and write
\[
\mu := \pi(f) = \int f\,d\pi = \mathbb E_q[r(X)f(X)].
\]
The SNIS estimator is
\[
\hat\pi^P(f)
=
\sum_{i=1}^P \bar w_i f(X_i)
=
\frac{\sum_{i=1}^P r_i f(X_i)}{S}.
\]
Hence
\[
\hat\pi^P(f)-\mu
=
\frac{\sum_{i=1}^P r_i(f(X_i)-\mu)}{S}.
\]

Define
\[
Z_i := r_i(f(X_i)-\mu),
\qquad
N := \sum_{i=1}^P Z_i.
\]
Then
\[
\hat\pi^P(f)-\mu = \frac{N}{S}.
\]

First, $\mathbb E[Z_i]=0$, because
\[
\mathbb E[Z_i]
=
\mathbb E_q[r(X)(f(X)-\mu)]
=
\mathbb E_q[r(X)f(X)] - \mu\,\mathbb E_q[r(X)]
=
\mu-\mu
=
0.
\]
Also, since $\|f\|_\infty\le 1$, we have $|\mu|\le 1$, hence
$|f(X_i)-\mu|\le 2$. Therefore
\[
Z_i^2 = r_i^2(f(X_i)-\mu)^2 \le 4 r_i^2,
\]
and so
\[
\mathbb E[Z_i^2]
\le
4\mathbb E_q[r(X)^2]
=
4(1+\chi^2(\pi\|q)).
\]
Because the $Z_i$ are independent and mean zero,
\[
\mathbb E[N^2]
=
\sum_{i=1}^P \mathbb E[Z_i^2]
\le
4P(1+\chi^2(\pi\|q)).
\]
By Cauchy--Schwarz,
\[
\mathbb E|N|
\le
\sqrt{\mathbb E[N^2]}
\le
2\sqrt{P}\,\sqrt{1+\chi^2(\pi\|q)}.
\]

Next, since $S=\sum_{i=1}^P r_i$ and $\mathbb E_q[r]=1$,
\[
\mathbb E[S]=P.
\]
Also,
\[
\mathrm{Var}(r)=\mathbb E_q[r^2]-1=\chi^2(\pi\|q),
\]
so
\[
\mathrm{Var}(S)=P\,\chi^2(\pi\|q).
\]
Let
\[
E:=\{S>P/2\}.
\]
By Chebyshev's inequality,
\[
\mathbb P(E^c)
=
\mathbb P\!\left(S\le \frac{P}{2}\right)
=
\mathbb P\!\left(S-P\le -\frac{P}{2}\right)
\le
\frac{4\,\mathrm{Var}(S)}{P^2}
=
\frac{4\chi^2(\pi\|q)}{P}.
\]

We now split on $E$:
\[
\mathbb E\left|\hat\pi^P(f)-\mu\right|
=
\mathbb E\left[\left|\frac{N}{S}\right|\mathbf 1_E\right]
+
\mathbb E\left[\left|\frac{N}{S}\right|\mathbf 1_{E^c}\right].
\]
On $E$, we have $1/S \le 2/P$, hence
\[
\mathbb E\left[\left|\frac{N}{S}\right|\mathbf 1_E\right]
\le
\frac{2}{P}\mathbb E|N|
\le
\frac{4}{\sqrt P}\sqrt{1+\chi^2(\pi\|q)}.
\]
On $E^c$, we use the deterministic bound
\[
|N|
=
\left|
\sum_{i=1}^P r_i(f(X_i)-\mu)
\right|
\le
\sum_{i=1}^P r_i |f(X_i)-\mu|
\le
2\sum_{i=1}^P r_i
=
2S,
\]
which implies $|N/S|\le 2$. Therefore
\[
\mathbb E\left[\left|\frac{N}{S}\right|\mathbf 1_{E^c}\right]
\le
2\,\mathbb P(E^c)
\le
\frac{8\chi^2(\pi\|q)}{P}.
\]
Combining the two bounds proves \eqref{eq:bounded-test-bound}.
\end{proof}

\subsection{Proof of Proposition~\ref{prop:apf-variance-reduction} 
(Variance reduction from future-value selection)}
\label{app:proof-variance-reduction}

\begin{proof}
By definition of the $\chi^2$-divergence,
\[
  1 + \chi^2(\pi_j\|\tilde q_j)
  = \sum_{b_j}\frac{\pi_j(b_j)^2}{\tilde q_j(b_j)}.
\]
Let $C := \sum_{b'} q_j(b')\,\psi(b') = \mathbb{E}_{q_j}[\psi(b_j)]$ act as the normalizer constant. Substituting the definition of the effective proposal $\tilde q_j(b_j) = q_j(b_j)\,\psi(b_j) / C$ and the ratio $r(b_j) = \pi_j(b_j)/q_j(b_j)$, we obtain:
\begin{align*}
  1 + \chi^2(\pi_j\|\tilde q_j)
  &= \sum_{b_j}\frac{\pi_j(b_j)^2\,C}{q_j(b_j)\,\psi(b_j)} \\
  &= C\sum_{b_j} q_j(b_j)\,\frac{r(b_j)^2}{\psi(b_j)} \\
  &= \mathbb{E}_{q_j}[\psi(b_j)]\;\mathbb{E}_{q_j}\!\!\left[\frac{r(b_j)^2}{\psi(b_j)}\right].
\end{align*}
We now frame this product of expectations using the general definition of covariance. For any two random variables $U$ and $W$ under a probability distribution $q_j$, it holds that $\operatorname{Cov}_{q_j}(U, W) = \mathbb{E}_{q_j}[UW] - \mathbb{E}_{q_j}[U]\,\mathbb{E}_{q_j}[W]$. 

Setting $U := \psi(b_j)$ and $W := r(b_j)^2 / \psi(b_j)$, we observe that their product trivially simplifies inside the expectation to $UW = r(b_j)^2$. Therefore:
\[
  \operatorname{Cov}_{q_j}\!\!\left(\psi(b_j),\; \frac{r(b_j)^2}{\psi(b_j)}\right)
  \;=\; \mathbb{E}_{q_j}[r(b_j)^2] \;-\; \mathbb{E}_{q_j}[\psi(b_j)]\;\mathbb{E}_{q_j}\!\!\left[\frac{r(b_j)^2}{\psi(b_j)}\right].
\]
Rearranging this equation immediately yields Identity~\eqref{eq:chi2-identity}:
\[
   1 + \chi^2(\pi_j\|\tilde q_j) \;=\; \mathbb{E}_{q_j}[r(b_j)^2] \;-\; \operatorname{Cov}_{q_j}\!\!\left(\psi(b_j),\; \frac{r(b_j)^2}{\psi(b_j)}\right).
\]
In the baseline scenario without a future-value selection potential ($\psi \equiv 1$), the covariance term is trivially $0$ (since variance/covariance with a constant is zero), evaluating to exactly $1 + \chi^2(\pi_j\|q_j) = \mathbb{E}_{q_j}[r(b_j)^2]$. By substituting this definition into the identity, the difference between the divergences is precisely the isolated covariance:
\[
   \chi^2(\pi_j\|\tilde q_j) \;=\; \chi^2(\pi_j\|q_j) \;-\; \operatorname{Cov}_{q_j}\!\!\left(\psi(b_j),\; \frac{r(b_j)^2}{\psi(b_j)}\right).
\]
Hence, a strict variance reduction $\chi^2(\pi_j\|\tilde q_j) < \chi^2(\pi_j\|q_j)$ occurs if and only if the covariance is strictly positive.

Finally, in the ideal limit where $\psi(b_j) \propto r(b_j)$, we have $r(b_j) = k\psi(b_j)$ for some scalar constant $k$. Then $W = k^2\psi(b_j) \propto U$. The positive correlation is maximized asymptotically, and directly applying the substitution into $\tilde q_j$ provides $\tilde q_j(b_j) = k q_j(b_j)\psi(b_j) / (kC) = \pi_j(b_j)$, rendering the target and effective proposal equal and trivially giving $\chi^2(\pi_j\|\tilde q_j) = 0$.
\end{proof}

\subsection{Corollary~\ref{cor:minimal-apps-is}: Minimal APPS as 
proposal-corrected importance sampling}
\label{app:base-apps}
\begin{corollary}[\rededit{Minimal APPS as proposal-corrected blockwise importance sampling}]
\label{cor:minimal-apps-is}
\rededit{Consider a sequential block proposal}
\[
\rededit{q(b_{1:J}\mid x)
=
\prod_{j=1}^J q_j(b_j\mid x,b_{<j}).}
\]
\rededit{Then the cumulative importance weight produced by the proposal-corrected update}
\[
\rededit{\Delta \log w_j
=
\alpha \log p(b_j\mid x,b_{<j})
-
\log q_j(b_j\mid x,b_{<j})}
\]
\rededit{satisfies}
\begin{equation}
\rededit{W_J(b_{1:J})
:=
\prod_{j=1}^J
\frac{p(b_j\mid x,b_{<j})^\alpha}{q_j(b_j\mid x,b_{<j})}
=
\frac{\gamma_J(b_{1:J}\mid x)}{q(b_{1:J}\mid x)}.}
\label{eq:full-importance-ratio}
\end{equation}
\rededit{Hence the minimal APPS sampler defines the standard blockwise importance-sampling construction for the sequence-level power target \eqref{eq:full-power-target}; with standard ESS-gated resampling, this becomes the corresponding blockwise SMC approximation.}
\end{corollary}

\begin{proof}
Summing the incremental log-weights gives
\[
\sum_{j=1}^J \Delta\log w_j
=
\sum_{j=1}^J
\log
\frac{p(b_j\mid x,b_{<j})^\alpha}{q_j(b_j\mid x,b_{<j})}
=
\log
\prod_{j=1}^J
\frac{p(b_j\mid x,b_{<j})^\alpha}{q_j(b_j\mid x,b_{<j})}.
\]
Exponentiating yields
\[
W_J(b_{1:J})
=
\prod_{j=1}^J
\frac{p(b_j\mid x,b_{<j})^\alpha}{q_j(b_j\mid x,b_{<j})}.
\]
Using the definitions of $\gamma_J$ and the sequential proposal factorization,
\[
\gamma_J(b_{1:J}\mid x)
=
\prod_{j=1}^J p(b_j\mid x,b_{<j})^\alpha,
\qquad
q(b_{1:J}\mid x)
=
\prod_{j=1}^J q_j(b_j\mid x,b_{<j}),
\]
which proves \eqref{eq:full-importance-ratio}.
\end{proof}

\paragraph{Discussion and scope.}
\rededit{Theorem~\ref{thm:future-value-factorization} is the conceptual centerpiece: it shows that future value is exactly the missing factor in the true block-boundary marginal. Theorem~\ref{thm:bounded-test-snis} then gives a correct finite-$P$ approximation guarantee for bounded observables, avoiding a support-free total-variation claim over all measurable sets. Corollary~\ref{cor:minimal-apps-is} formalizes that the proposal-corrected minimal APPS backbone is already a valid blockwise importance-sampling / SMC construction for the sequence-level power target.}

\paragraph{Important note on APF selection potentials.}
\rededit{The target-preservation statement in Corollary~\ref{cor:minimal-apps-is} applies to the minimal proposal-corrected sampler. If one resamples using modified selection weights proportional to $w_j^{(i)}\psi_j^{(i)\eta}$, then preserving the same underlying target generally requires an explicit auxiliary-SMC correction. Without such a correction, the modified resampling law should be interpreted as a finite-budget selection heuristic rather than an exact target-preserving transformation.}

\section{Rollout and learned selection potentials}
\label{app:learned-apf}

We describe the two practical realizations of the selection potential in
\eqref{eq:pps-apf-resample}. In rollout APF, the future-value signal is computed directly from short
lookahead launched from the current prefixes. In learned APF, the same signal is amortized by a
lightweight predictor evaluated at the same resampling boundaries while keeping the base language
model fixed.

\subsection{Boundary-level supervision}
Rollout supervision is collected at resampling boundaries. For each candidate prefix evaluated at
boundary $j$, we store
\[
\bigl(h_j^{(i)},\; y_j^{(i)},\; g_j^{(i)}\bigr),
\qquad
y_j^{(i)} := \log \psi_j^{(i)},
\]
where $h_j^{(i)}\in\mathbb{R}^d$ is the boundary representation, $y_j^{(i)}$ is the rollout-based
selection potential, and $g_j^{(i)}$ identifies the corresponding boundary group. All examples with
the same group identifier arise from the same resampling decision.

The group structure is fundamental. In a particle system, future value matters only through how
selection redistributes representation within a boundary group. We therefore preserve groups across
shard merging, with an ablated filter for degenerate groups if needed, and split train/validation
data by group. To reduce avoidable variance, rollout targets are generated with a deterministic
seeding policy during data construction.

\subsection{Effective training target}
The learned head is trained to predict the signal actually used at decode time. Let $g$ denote a
boundary group and let $y^{(i)}$ be the raw rollout score for candidate $i\in g$. The effective
target is
\begin{equation}
\tilde y^{(i)}
=
\eta \cdot \mathrm{clip}\!\left(
y^{(i)} - \frac{1}{|g|}\sum_{k\in g} y^{(k)}
\right),
\label{eq:learned-apf-effective-target}
\end{equation}
with the obvious simplifications when centering or clipping is disabled.

This centering makes the target explicitly boundary-relative: additive offsets within a group are
irrelevant, because selection depends only on how prefixes compete for representation at the same
resampling boundary. Equivalently, the learned target is defined only up to within-boundary
competition: what matters is not the absolute scale of future value, but how it tilts selection
probabilities and expected offspring within the current population.

We considered two target parameterizations. \emph{Standardized-target distillation} predicts globally
normalized rollout scores and relies on the loss to recover boundary structure. Our main formulation,
\emph{decode-aligned distillation}, predicts the effective target in
\eqref{eq:learned-apf-effective-target} directly. Unless noted otherwise, the learned-APF results in
the main text use decode-aligned distillation.

\subsection{Head and objective}
The learned selection potential is produced by a lightweight multilayer perceptron
\[
f_\theta : \mathbb{R}^d \to \mathbb{R},
\]
applied to the boundary representation $h_j^{(i)}$. Since the learned potential is used only through
selection within a boundary, the objective emphasizes relative fidelity within groups rather than
absolute calibration across unrelated boundaries.

For decode-aligned distillation, let $\hat y^{(i)}$ denote the student prediction after applying the
same groupwise centering, clipping, and scaling used in decoding. The training loss is a weighted sum
of:
\begin{enumerate}
\item a smooth-$L_1$ regression term between $\hat y$ and $\tilde y$;
\item a group-centered regression term;
\item a listwise distillation term matching the rollout and learned  softmax distributions within each boundary;
\item a within-group pairwise ranking loss; and
\item a top-1 cross-entropy term matching the winning prefix selected by rollout APF.
\end{enumerate}
Group weights are derived from the rollout signal, assigning larger weight to boundaries with larger
within-group spread or winner--runner-up margin. These terms all target the same object from
different angles: not the absolute value of a prefix in isolation, but its role in the boundary-level
competition for survival and offspring.

\subsection{Optimization and model selection}
We train with AdamW and cosine annealing. Minibatches are formed from whole groups rather than
independent rows, since the structured losses are defined over candidates from the same boundary. We
also maintain an exponential moving average of parameters for evaluation.

Held-out evaluation is group-aware. In addition to pointwise error, we compute group top-1
agreement, group pairwise accuracy, and Pearson correlation with rollout effective scores.
Checkpoints are selected primarily by group top-1 agreement, with group pairwise accuracy used as a
tie-breaker. This matches the intended role of the learned head: improving resampling decisions
rather than fitting raw rollout values.

\subsection{Summary}
The learned APF head is trained to imitate the effective rollout signal used by APPS at resampling
time. The construction is group-aware throughout: supervision is collected at resampling boundaries,
objectives model competition within a boundary, batching preserves boundary structure, and checkpoint
selection is based on boundary-level selection quality.

\section{Blockwise execution, cache reuse, and dynamic allocation}
\label{app:blockwise-execution}
\noindent\textbf{Blockwise execution and KV reuse.}
APPS is formulated at the block level but executed through a streaming KV-cache path. The model is evaluated once on the prompt, after which decoding proceeds token by token while maintaining one KV cache per particle; at resampling boundaries, caches are reordered according to the selected ancestors, analogous to beam search. The method is thus blockwise in logic but streaming in execution.

\smallskip
\subsection{Finite-budget population reallocation}
\label{app:dynamic-allocation}

Under dynamic allocation, APPS updates the active population size only at block boundaries. The underlying idea is simple: boundaries that remain unresolved should retain more particle mass, while collapsed boundaries should release budget.

Let $\{\log \tilde w_j^{(i)}\}_{i=1}^{P_j}$ denote the selection scores at boundary $j$, after any rollout or learned \rededit{selection potential} has been incorporated. We summarize boundary uncertainty by an ambiguity score $a_j \in [0,1]$. In practice, we use a simple signal derived either from the top-one/top-two gap or from the entropy of the normalized selection weights. A representative entropy-based form is
\[
\bar \pi_j^{(i)}
=
\frac{\exp(\log \tilde w_j^{(i)})}
     {\sum_{k=1}^{P_j}\exp(\log \tilde w_j^{(k)})},
\qquad
a_j
=
\frac{-\sum_{i=1}^{P_j}\bar \pi_j^{(i)} \log \bar \pi_j^{(i)}}
     {\log P_j}.
\]
Large $a_j$ indicates a diffuse, unresolved population; small $a_j$ indicates that selection mass is already concentrated.

The next active population size is then chosen by a monotone budget-allocation rule,
\[
P_{j+1}
=
g(a_j; P_{\min}, P_{\max}),
\qquad
P_{\min} \le P_{j+1} \le P_{\max},
\]
where $g$ increases with ambiguity. In our implementation, this is realized as a clipped linear map,
\[
P_{j+1}
=
\mathrm{clip}\!\left(
P_{\min}
+
\left\lfloor
(P_{\max}-P_{\min})\, a_j
\right\rceil,
\;
P_{\min},
\;
P_{\max}
\right),
\]
and the updated budget is realized at the next resampling step. Thus dynamic allocation does not alter the proposal or weighting rule; it changes only how much particle mass remains active at subsequent boundaries.

To characterize this mechanism empirically, we report several boundary-level diagnostics: the average active population over blocks,
\[
\bar P = \frac{1}{J}\sum_{j=1}^J P_j,
\]
the fraction of blocks at minimum population,
\[
\frac{1}{J}\sum_{j=1}^J \mathbf{1}\{P_j = P_{\min}\},
\]
the mean ambiguity fraction,
\[
\bar a = \frac{1}{J}\sum_{j=1}^J a_j,
\]
the mean absolute change in active population size,
\[
\frac{1}{J-1}\sum_{j=1}^{J-1} |P_{j+1}-P_j|,
\]
the number of allocation increases and decreases, and the average number of unique ancestors after resampling. Together, these diagnostics distinguish runs that are merely faster from runs that genuinely reallocate finite compute toward unresolved regions without collapsing useful diversity.

\smallskip
\noindent\textbf{Scores across resampling events.}
We maintain two notions of score. The first is a local selection score, derived from the incremental weight since the last resampling step and used to form the resampling distribution at the current boundary. The second is a cumulative ancestry score, copied through sampled parent indices and kept comparable across resampling events for final particle selection. This keeps local resampling decisions numerically stable while retaining a trajectory-level score for final selection across resampling events.

\subsection{Further Runtime--Accuracy Analysis under Dynamic Allocation}
\label{app:further-runtime-scatter}

Figure~\ref{fig:rollout_learned_dynamic_tradeoff} provides a complementary runtime--accuracy view of the APF-guided variants under dynamic allocation. The plot focuses on the operating regions of learned APF and rollout APF, showing how the two future-value estimators trade off online cost and final accuracy.

\begin{figure*}[t]
    \centering
    \includegraphics[width=\textwidth]{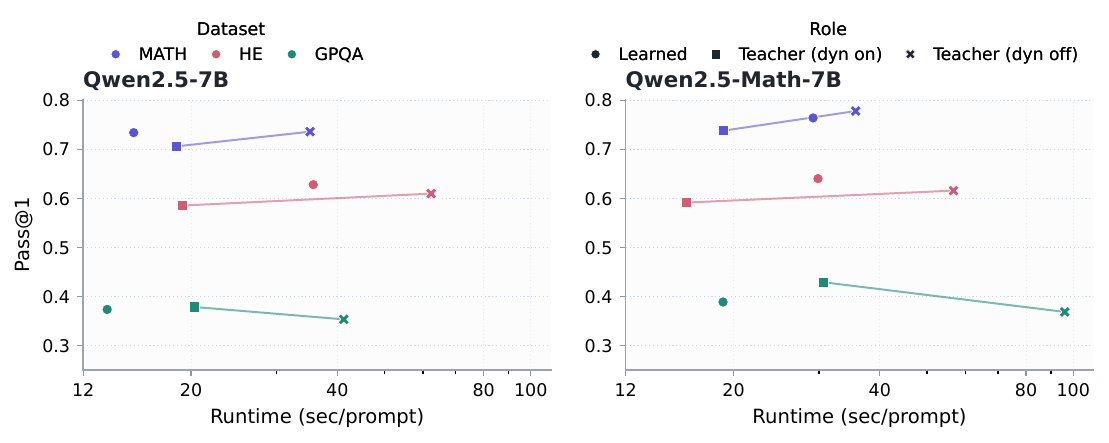}
    \caption{
    \textbf{Wall-clock runtime--accuracy trade-offs for learned and rollout APF.}
    Each point is a completed full-benchmark run, plotted by runtime per prompt and \texttt{pass@1}.
    Dynamic allocation acts primarily as a compute-control mechanism, reducing the online cost of APF-guided decoding by adapting the active particle population during generation.
    Learned APF provides the strongest overall speed--accuracy trade-off on MATH500 and HumanEval, while on GPQA, rollout APF reaches higher \texttt{pass@1} at higher wall-clock cost.
    Overall, learned APF is the more efficient future-value estimator, whereas rollout APF can be more accurate when its online estimates justify the additional compute.
    }
    \label{fig:rollout_learned_dynamic_tradeoff}
\end{figure*}

\subsection{Main Hyperparameter Settings}
\label{app:hparams}

Table~\ref{tab:apf-main-hparams} summarizes the main decoding hyperparameters used for APPS in our experiments. We align the evaluation protocol with Scalable Power Sampling (SPS)~\citep{ji2026scalable} wherever possible, including the base models, datasets, prompts, evaluators, maximum generation length, sampling temperature, and particle budget. In particular, SPS uses $\alpha=4$, temperature $1/\alpha=0.25$, maximum generation length $T_{\max}=3072$ with EOS early stopping, and reports Best-of-$N$ with $N=32$ as well as MCMC power sampling with block size $B=192$ and $N_{\mathrm{MCMC}}=10$. We use the same global generation cap and temperature, and evaluate APPS at the same main particle budget $P=32$.

For the main APPS results, we set $P=32$, sampling temperature $\tau=0.25$, and maximum generation length $T_{\max}=3072$ across all benchmarks, terminating generation early upon EOS. For APF-guided variants, we fix the test-time APF strength to $\eta=0.5$; for p-only APPS, this parameter is inactive because no auxiliary future-value selection potential is used. Rollout APF uses two short lookahead rollouts per APF boundary, each with horizon $H_{\mathrm{roll}}=16$. We use dataset-specific block sizes: $B=192$ for MATH500, $B=16$ for HumanEval, and $B=64$ for GPQA.

Unless explicitly noted, APPS variants within a benchmark are evaluated under the same base model, prompt format, evaluator, generation cap, EOS stopping rule, particle-budget cap, and dataset-specific block schedule. The variants are distinguished by their boundary-level selection signal: absent in p-only APPS, rollout-estimated in rollout APF, and learned in learned APF.

\begin{table}[t]
\centering
\small
\setlength{\tabcolsep}{4.5pt}
\begin{tabular}{lcccccccc}
\toprule
Dataset
& Particles $P$
& Power $\alpha$
& Temp. $\tau$
& $T_{\max}$
& APF $\eta$
& Block $B$
& Rollouts
& Horizon $H_{\mathrm{roll}}$ \\
\midrule
MATH500   & 32 & 4 & 0.25 & 3072 & 0.5 & 192 & 2 & 16 \\
HumanEval & 32 & 4 & 0.25 & 3072 & 0.5 & 16  & 2 & 16 \\
GPQA      & 32 & 4 & 0.25 & 3072 & 0.5 & 64  & 2 & 16 \\
\bottomrule
\end{tabular}

\vspace{0.4em}
\caption{
\textbf{Main APPS decoding hyperparameters.}
We use $P=32$ particles, power parameter $\alpha=4$, sampling temperature $\tau=1/\alpha=0.25$, and maximum generation length $T_{\max}=3072$ for the main APPS results, with generation terminating early upon EOS. These global settings follow the SPS evaluation protocol~\citep{ji2026scalable}. For APF-guided variants, the test-time APF strength is fixed to $\eta=0.5$; for p-only APPS, $\eta$ is inactive because no auxiliary future-value selection potential is applied. Rollout APF uses two short lookahead rollouts per APF boundary with horizon $H_{\mathrm{roll}}=16$. The block size $B$ is dataset-specific. The main controlled difference across APPS variants is the boundary-level selection signal: p-only, rollout APF, or learned APF.
}
\label{tab:apf-main-hparams}
\end{table}

For learned APF, we train a lightweight boundary-value head to approximate the APF selection potential from rollout-derived supervision. Training examples are constructed at APF decision boundaries, with train--validation splits performed by boundary group to avoid leakage across closely related prefixes. The head is trained with AdamW and cosine learning-rate annealing using group-preserving minibatches of size $512$, dropout $0.10$, weight decay $5\times 10^{-2}$, and an exponential moving average of parameters with decay $0.995$. We train for up to $60$ epochs with early stopping patience $12$ and minimum validation improvement $10^{-3}$.

For decode-aligned distillation, rollout-derived targets are smoothed using $\eta_{\mathrm{distill}}=0.4$ before fitting the learned APF potential. This coefficient is used only for target shaping during distillation and is distinct from the test-time APF selection strength $\eta=0.5$. At test time, the learned head serves only as a boundary-level selection potential; the base model, propagation sampler, evaluator, generation cap, and stopping rule are unchanged.

\subsection{Hyperparameter Sensitivity}
\label{app:hparam_sensitivity}

Figure~\ref{fig:qwen_math_hparam_sensitivity} reports a controlled sensitivity sweep over block size $B$ and APF strength $\eta$ for Qwen2.5-Math-7B. For clarity of interpretation, dynamic allocation is disabled in this sweep so that each cell reflects the effect of $B$ and $\eta$ under the same fixed particle budget. The results suggest that APF-guided APPS is not driven by a single fragile setting: MATH500 is stable across a broad range of choices, while HumanEval and GPQA show stronger task-specific dependence on $B$ and $\eta$. Rollout APF can achieve the strongest individual cells but is more sensitive to $\eta$, whereas learned APF is generally smoother, consistent with its role as an amortized selection potential.

\begin{figure}[t]
    \centering
    \includegraphics[width=\textwidth]{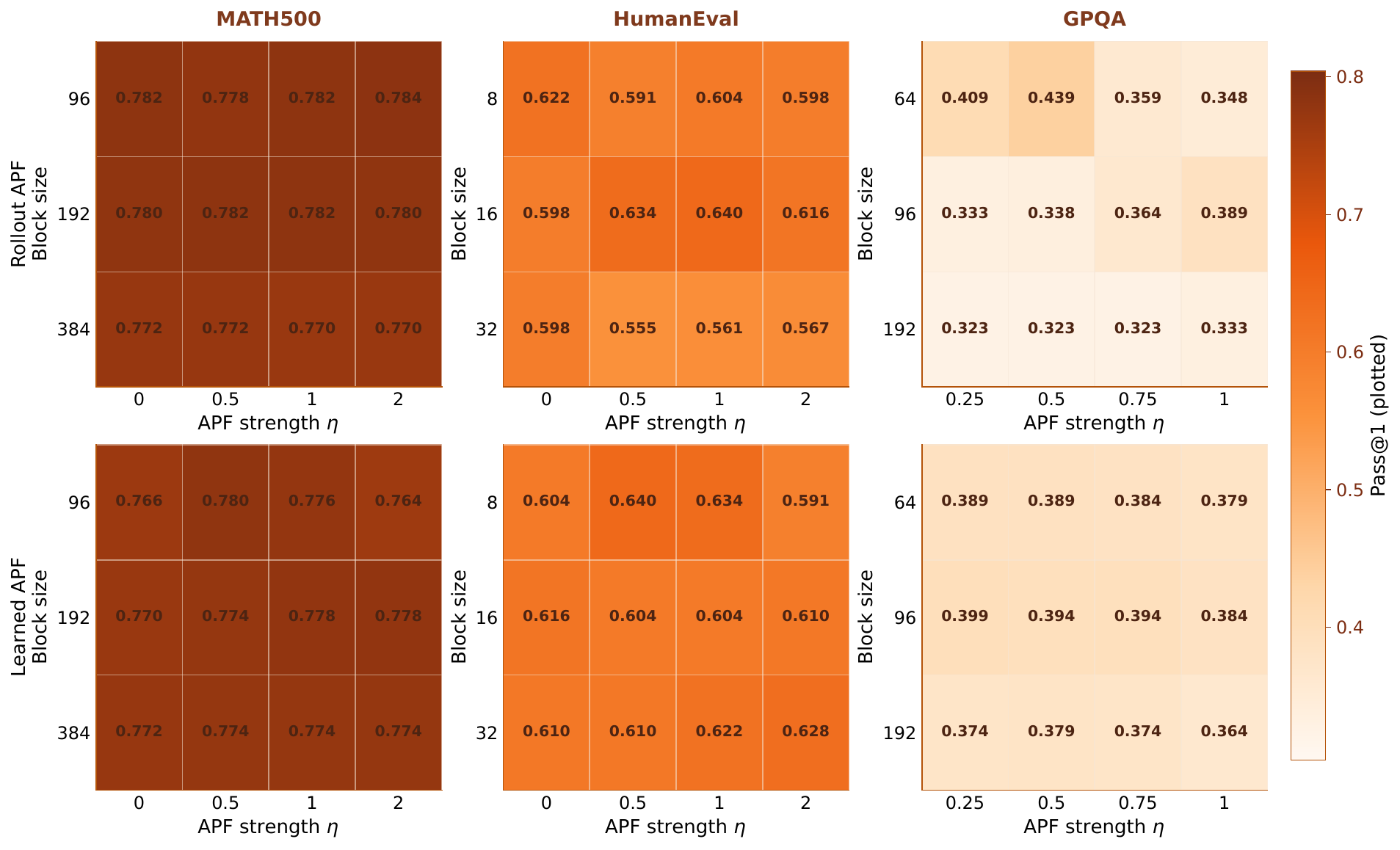}
    \caption{
    \textbf{Hyperparameter sensitivity of APF-guided APPS on Qwen2.5-Math-7B.}
    We sweep block size $B$ and APF strength $\eta$ for rollout APF and learned APF on MATH500, HumanEval, and GPQA, reporting \texttt{pass@1} in each cell.
    For clarity of interpretation, dynamic allocation is disabled so that all cells are evaluated under the same fixed particle budget; all other decoding settings are held fixed within each benchmark.
    MATH500 is relatively stable across the tested settings, while HumanEval and GPQA exhibit stronger task-specific dependence on $B$ and $\eta$.
    Rollout APF obtains the strongest individual cells in several settings but is more sensitive to $\eta$, reflecting the variance of finite-sample, finite-horizon online lookahead.
    Learned APF yields smoother performance across the grid, consistent with an amortized selection potential that reduces online variance and cost.
    Overall, APF-guided APPS remains competitive across a range of nearby hyperparameter choices, with learned APF showing smoother behavior across the tested grid.
    }
    \label{fig:qwen_math_hparam_sensitivity}
\end{figure}

\end{document}